\documentclass{article}

    \usepackage[final,nonatbib]{neurips_2021}

\usepackage[utf8]{inputenc} 
\usepackage[T1]{fontenc}    
\usepackage{hyperref}       
\usepackage{url}            
\usepackage{booktabs}       
\usepackage{amsfonts}       
\usepackage{nicefrac}       
\usepackage{microtype}      
\usepackage{xcolor}         

\usepackage{times}

\usepackage{soul}
\usepackage{url}
\usepackage[small]{caption}
\usepackage{cite}
\usepackage{listings}
\usepackage{amsmath,bbm,amsfonts,amssymb,dsfont,amsthm}
\usepackage{graphicx}
\usepackage{xcolor}

\newtheorem{proposition}{Proposition}

\usepackage{amsmath,amsfonts,bm}

\def\eqref#1{equation~\ref{#1}}

\def\1{\bm{1}}

\def\vg{{\bm{g}}}

\def\vx{{\bm{x}}}
\def\vy{{\bm{y}}}

\DeclareMathAlphabet{\mathsfit}{\encodingdefault}{\sfdefault}{m}{sl}
\SetMathAlphabet{\mathsfit}{bold}{\encodingdefault}{\sfdefault}{bx}{n}

\usepackage{nicefrac, subcaption}
\usepackage{tikz}
\usetikzlibrary{
  matrix,
  graphs,
  graphs.standard,
  positioning,chains,fit,shapes,calc,
  decorations.pathreplacing
}
\newcommand{\sort}{\mathrm{sort}}

\newcommand{\name}{PiRank}

\newif\ifanonymous
\anonymousfalse

\newif\ifcomments
\commentsfalse
\ifcomments
\newcommand{\s}[1]{{\color{red} SE: #1}}
\newcommand{\ag}[1]{{\color{blue} AG: #1}}
\newcommand{\robin}[1]{{\color{orange} RS to BC: #1}}
\newcommand{\bc}[1]{{\color{purple} BC: #1}}
\else
\newcommand{\s}[1]{}
\newcommand{\ag}[1]{}
\newcommand{\robin}[1]{}
\newcommand{\bc}[1]{}
\fi

\AtBeginDocument{%
  \providecommand\BibTeX{{%
    \normalfont B\kern-0.5em{\scshape i\kern-0.25em b}\kern-0.8em\TeX}}}

\ifanonymous
\else
\author{
Robin Swezey$^1$\thanks{This work was done prior to joining Amazon.}\and
\textbf{Aditya Grover}$^{2,3}$\and
\textbf{Bruno Charron}$^1{}^*$\and
\textbf{Stefano Ermon}$^4$ \and
$^1$Amazon\\
$^2$University of California, Los Angeles\\
$^3$Facebook AI Research\\
$^4$Stanford University\\
\texttt{\{rswezey, bcharron\}@acm.org},
\texttt{adityag@cs.ucla.edu},
\texttt{ermon@cs.stanford.edu}
}
\fi

\title{PiRank: Scalable Learning To Rank via Differentiable Sorting}

\begin{document}

\maketitle

\begin{abstract}
A key challenge with machine learning approaches for ranking is the gap between the performance metrics of interest and the surrogate loss functions that can be optimized with gradient-based methods.
This gap arises because ranking metrics typically involve a sorting operation which is not differentiable w.r.t. the model parameters.
Prior works have proposed surrogates that are loosely related to ranking metrics or simple smoothed versions thereof, and often fail to scale to real-world applications.
We propose \name{}, a new class of differentiable surrogates for ranking, which employ a continuous, temperature-controlled relaxation to the sorting operator based on NeuralSort~\cite{grover2019stochastic}. 
We show that \name{} exactly recovers the desired metrics in the limit of zero temperature  and further propose a divide-and-conquer extension that scales favorably to large list sizes, both in theory and practice.
Empirically, we demonstrate the role of larger list sizes during training and show that \name{} 
significantly improves over comparable approaches on publicly available internet-scale learning-to-rank
benchmarks.
\end{abstract}

\section{Introduction}

\label{sec:introduction}

The goal of Learning-To-Rank (LTR) models is to rank a set of candidate items for any given search query according to a preference criterion~\cite{liu2009learning}.
The preference over items is specified via relevance labels for each candidate.
The fundamental difficulty in LTR is that the downstream metrics of interest such as normalized discounted cumulative gain (NDCG) and average relevance position (ARP) depend on the ranks
induced by the model.
These ranks are not differentiable with respect to the model parameters, so the metrics cannot be optimized directly via gradient-based methods.

To resolve the above challenge, a popular class of LTR approaches map items to real-valued scores and then define surrogate loss functions that operate directly on these scores. 
Surrogate loss functions, in turn, can belong to one of three types.
LTR models optimized via \textit{pointwise} surrogates
\cite{cossock2006subset,li2008mcrank,crammer2002pranking,shashua2003ranking} 
cast ranking as a regression/classification problem, wherein the labels of items are given by their individual relevance labels.
Such approaches do not directly account for any inter-dependencies across item rankings.
\textit{Pairwise} surrogate losses 
\cite{herbrich2000large,freund2003efficient,Burges2005,burges2010ranknet,zheng2008general,cao2006adapting,Burges2007,wu2009smoothing}
can be decomposed into terms that involve scores of pairs of items in a list and their relative ordering. 
Finally, \textit{listwise} surrogate losses
\cite{cao2007learning,xia2008listwise,xu2007adarank,yue2007support,taylor2008softrank}
are defined with respect to scores for an entire ranked list.
For many prior surrogate losses, especially those used for listwise approaches, the functional form is inspired via downstream ranking metrics, such as NDCG. However, the connection is loose or heuristically driven. For instance, SoftRank~\cite{taylor2008softrank,wu2009smoothing} introduces a Gaussian distribution over scores, which in turn defines a distribution over ranks and the surrogate is the expected NDCG w.r.t. this rank distribution.

We propose \name{}, a listwise approach where the scores are learned via deep neural networks and the surrogate loss is obtained via a differentiable relaxation to the sorting operator.
In particular, we choose as building block the temperature-controlled  NeuralSort~\cite{grover2019stochastic} relaxation for sorting and specialize it for 
commonly used ranking metrics such as NDCG and ARP. 
The resulting training objective for \name{} reduces to the exact ranking metric optimization in the limit of zero temperature and trades off bias for lower variance in the gradient estimates when the temperature is high.
Furthermore, \name{} scales to real-world industrial scenarios where the size of the item lists is very large but the ranking metrics of interest are determined by only a small set of top ranked items. 
Scaling is enabled by a novel divide-and-conquer strategy akin to merge sort, where we recursively apply the sorting relaxation to sub-lists of smaller size and propagate only the top items from each sub-list for further sorting.

Empirically, we benchmark \name{} against 5 competing methods on two of the largest publicly available LTR datasets: MSLR-WEB30K~\cite{Qin2013} and Yahoo! C14.
We find that \name{} is superior or competitive on 13 out of 16 ranking metrics and their variants, including 9 on which it is significantly superior to all baselines, and that it is able to scale to very large item lists.
We also provide several ablation experiments to understand the impact of various factors on performance.
To the best of our knowledge, this work is the first to analyze the importance of training list size on an LTR benchmark.
Finally, we provide an open-source implementation\ifanonymous{ }\else\footnote{https://github.com/ermongroup/pirank} \fi
based on TensorFlow Ranking~\cite{pasumarthi2019tf}.

\section{Background and Related Work}

The LTR setting considers a finite dataset consisting of $n$ triplets $D=\{q_i, \{\vx_{i,j}\}_{j=1}^L, \{y_{i,j}\}_{j=1}^L\}_{i=1}^n$. 
The $i$-th triplet consists of a query $q_i \in  \mathcal{Q}$, a list of $L$ candidate items represented as feature vectors $\vx_{i,j} \in \mathcal{X}$, and query-specific relevance labels $y_{i,j}$ for each item $j$.
The relevance labels $y_{i,j}$ can be binary, ordinal or real-valued for more fine-grained relevance. For generality, we focus on the real-valued setting.
Given a training dataset $D$, our goal is to 
learn a mapping from queries and itemsets to rankings.
A ranking $\pi$ is a list of unique indices from $\{1, 2, \dots, L\}$, or equivalently a permutation, such that $\pi_j$ is the index of the item ranked in $j$-th position.
Without loss of generality, we assume lower ranks (starting from 1) have higher relevance scores. 
This is typically achieved by learning a scoring function $f: \mathcal{Q} \times \mathcal{X}^L\to \mathbb{R}^L$ \s{i think subscripts shouldn't be there} \robin{Agreed?} \bc{fixed} that maps a query context and list of candidate items to $L$ scores. 
At test time, the candidate items are ranked by sorting their predicted scores in descending order.
The training of $f$ 
itself can be done by a suitable differentiable surrogate objective, which we discuss next.

\subsection{Surrogate Objectives for LTR}

In this section, we briefly summarize prominent LTR approaches with a representative loss function for each category of pointwise, pairwise or listwise surrogate losses.
We refer the reader to the excellent survey by \cite{liu2011learning} for a more extensive review.
Omitting the triplet index, we denote the relevance labels vector as $\vy \in \mathbb{R}^L$ and an LTR model's score vector obtained via the scoring function  $f$ as $\hat{\vy} \in \mathbb{R}^L$.

The simplest pointwise surrogate loss for ranking is the mean-squared error (MSE) between $\vy$ and $\hat{\vy}$:

\begin{align}
\hat{\ell}_{\mathrm{MSE}}(\boldsymbol{y}, \boldsymbol{\hat{y}})=\frac{1}{L}\sum_{i=1}^{L}\left(\hat{y}_i-y_{i}\right)^{2}.
\end{align}

As the example loss above shows, pointwise LTR approaches convert ranking into a regression problem over the relevance labels and do not account for the relationships between the candidate items.
Pairwise approaches seek to remedy this by considering loss terms depending in the predicted scores of pairs of items.
For example, the widely used RankNet~\cite{Burges2005} aims to minimize the number of inversions, or incorrect relative orderings between pairs of items in the predicted ranking. 
It does so by modeling the probability $\hat p_{i, i'}$ that the relevance of the $i$-th item 
is higher than that of 
the $i'$-th item 
as a logistic map of their score difference, for all candidate items $i$, $i'$.
The objective is then the cross-entropy:
\begin{align}\label{eq:ranknet}
\hat{\ell}_{\mathrm{RankNet}}(\boldsymbol{y}, \hat{\boldsymbol{y}})=-\sum_{i=1}^{L} \sum_{i'=1}^{L} \mathbbm{1}
\left(y_{i}>y_{i'}\right) 
\log \hat p_{i, i'}
\end{align}
\s{looks weird, doesn't seem to depend on yhat} \robin{Replace with Phat?} \bc{replaced $P_{i,i'}$ by $\hat p_{i,i'}$ not to be confused with a permutation matrix and added that it is a function of $\hat \vy$}
where $\mathbbm{1}(\cdot)$ denotes the indicator function
and $\hat p_{i,i'}$ is a function of $\hat\vy$.
Pairwise approaches effectively model relationships between pairs of items and generally perform better than pointwise approaches, but still manifest limitations on downstream metrics which consider rankings in the full list and not just pairs.
In fact, the larger the list of candidate items, the weaker these approaches tend to be: 
an error between the first and second item on a list is weighted the same in the RankNet loss as one between the last two items, despite the top items being of more importance in the LTR setting.

Listwise approaches learn from errors on the complete list.
LambdaRank~\cite{Burges2007} extends RankNet by assigning weights to every loss term from Eq.~\ref{eq:ranknet}:
\begin{align}\label{eq:lambdarank}
\hat{\ell}_{\mathrm{LambdaRank}}(\boldsymbol{y}, \hat{\boldsymbol{y}})=-\sum_{i=1}^{L} \sum_{i'=1}^{L} \Delta \ell_{\mathrm{NDCG}}(i, i')
\log \hat p_{i, i'}
\end{align}
with $\Delta \ell_{\mathrm{NDCG}}(i, i')$ the difference in the downstream metric NDCG (defined below) when swapping items $i$ and $i'$.\s{NDCG needs to be defined before you use it. or at least say "defined below".}\robin{Define the ranking metrics before the approaches to work backwards from them?}\bc{fixed}

\subsection{Ranking Metrics}
\label{sec:ir}

Downstream metrics operate directly on the predicted ranking $\hat\pi$ (obtained by sorting $\hat{\vy}$ in descending order)
and the true relevance labels $\vy$.
They differ from conventional metrics used for other supervised learning problems in explicitly weighting the loss for each item by a suitably choosen increasing function of its predicted rank.
For example, relevance position (RP)~\cite{zhu2004recall} multiplies the relevance labels with linearly increasing weights, and normalizes by the total relevance score for the query:
 \begin{align}\label{eq:ranked_precision}
\mathrm{RP}(\vy, \hat \pi)=\frac{\sum_{j=1}^{L} y_{\hat\pi_j} j}{\sum_{j=1}^{L} y_{j}}
 \end{align}
Averaging RP across the predictions made for all the queries in the test set gives the average relevance position (ARP) metric. Lower ARP signifies better performance.

A very common metric is the discounted cumulative gain (DCG)~\cite{jarvelin2002cumulated}.
 DCG computes the rescaled relevance of the $j$-th candidate by exponentiating its relevance label, and 
 further divides it by 
 the assigned log-ranking. 
 This model incentivizes ranking models to focus on elements with higher graded relevance scores:
 \begin{align}
 \label{eq:DCG}
\mathrm{DCG}(\vy, \hat\pi)=\sum_{j=1}^{L} \frac{2^{y_{\hat\pi_j}}-1}{\log _{2}(1+j)}
 \end{align}

A more common variant NDCG normalizes DCG by the maximum possible DCG attained via the optimal ranking $\pi^\ast$ (obtained by sorting $\vy$ in descending order):
 \begin{align}
 \label{eq:NDCG}
\mathrm{NDCG}(\vy, \hat\pi)=\frac{\mathrm{DCG}( \vy, \hat\pi)}{\mathrm{DCG}\left(\vy, \pi^{*}\right)} 
\end{align}
Higher DCG and NDCG signify better performance. 
Their truncated versions DCG@$k$ and NDCG@$k$ are defined by replacing $L$ with a cutoff $k$ in Eq.~\ref{eq:DCG} so metrics are computed on the top-$k$ items.

\section{Scalable and Differentiable Top-$k$ Ranking via \name{}}\label{sec:technical}

In \name{}, we seek to design a new class of surrogate objectives for ranking that address two key challenges with current LTR approaches.
The first challenge is the 
gap between the downstream ranking metric of interest (e.g., NDCG, ARP) that involve a non-differentiable sorting operator and the differentiable surrogate function being optimized.
The second challenge concerns the scalability w.r.t. the size of the candidate list $L$ for each query item.
Larger list sizes are standard in industrial applications but present computational and memory challenges for current approaches during both training and test-time inference.
Pairwise and listwise methods (or hybrids) typically scale quadratically in the list size $L$, the number of items to rank for each query.
Combining surrogates for truncated metrics, such as LambdaRank in Eq.~\ref{eq:lambdarank} with NDCG@$k$ has a reduced complexity of $O(kL)$ but comes at the cost of vanishing gradient signal from relevant entries below $k$ (see Figure~\ref{fig:gradients} for an illustration).
Soft versions of the truncation metrics, such as Approximate NDCG Loss~\cite{qin2010general}, can 
learn from all items but again scale quadratically with $L$ or do not take advantage of GPU acceleration\cite{blondel2020fast}.

\begin{figure}
    \centering
    \includegraphics[width=.38\textwidth]{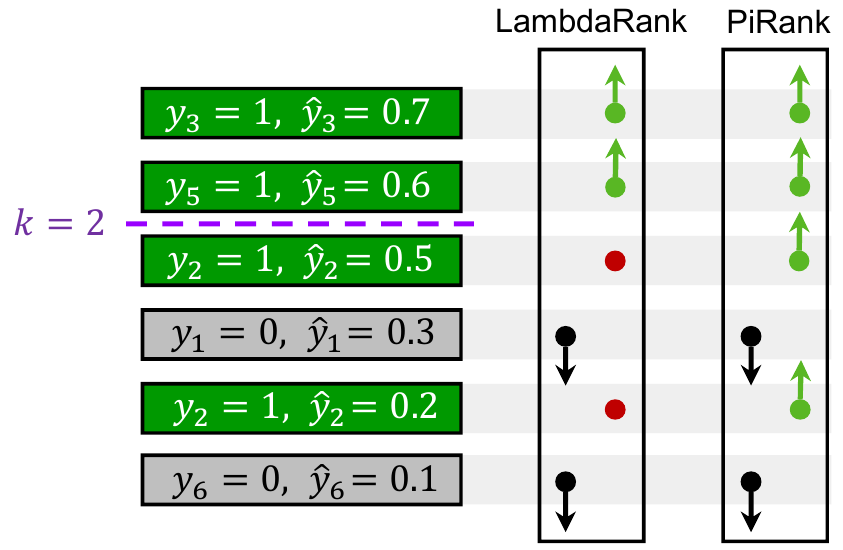}
    \caption{A set of items, green if relevant and gray otherwise, sorted by their score. Arrows show the sign of the loss derivative w.r.t. each item's predicted score for different methods (positive for black, negative for green and zero for red dots). Pairwise approaches weighted by differences in truncated ranking metrics, such as LambdaRank with NDCG@$k$, would put zero weights on the relevant items ranked below $k=2$, thus bypassing learning signal. In comparison, PiRank efficiently learns from all items even using a $k=2$ truncated loss.
    }
    \label{fig:gradients}
\end{figure}


As defined previously, a ranking $\pi$ is a list of indices equivalent to a permutation of $\{1, 2, \dots, L\}$.
The set of possible rankings can thus be seen as the symmetric group $\mathcal{S}_L$, of size $L!$.
Every permutation $\pi$ can be equivalently represented as a permutation matrix $P_\pi$, an $L \times L$ matrix such that its $(i, \pi_i)$-th entry is $1$ for all $i \in \{1,2,\cdots,L\}$ and the remaining entries are all $0$.
We define the sorting operator $\sort: \mathbb{R}^L \rightarrow \mathcal{S}_L$ as a map of an $L$-dimensional input vector to the permutation that corresponds to the descending ordering of the vector components.
Prior work in relaxing the sorting operator is based on relaxation of its output, either in the form of rankings~\cite{taylor2008softrank,chapelle2010gradient,blondel2020fast}, or permutation matrices~\cite{grover2019stochastic,adams2011ranking,Mena2018,Cuturi2019}. 
PiRank first applies the latter kind of relaxations to the ranking problem by introducing a new class of relaxed LTR metrics, then introduces a new relaxation that is particularly suited to these metrics.

\subsection{Relaxed Ranking Metrics}
\label{sec:relaxndcg}

We denote an LTR model by $f_\theta$ (e.g., deep neural network) with parameters $\theta$.
The model outputs a vector of $L$ scores $\hat{\vy}=f_\theta(q, \vx_1, \dots, \vx_L)$ for a query $q$ and $L$ candidate elements $\{\vx_i\}_{i=1}^L$.
We first consider the NDCG target metric.
In Eq.~\ref{eq:NDCG}, 
the numerator $\mathrm{DCG}(\vy, \hat\pi)$ involves computing $\hat\pi= \mathrm{sort}(\hat{\vy})$
which is non-differentiable w.r.t. $\theta$.
Let $\vg$ denote the column vector of graded relevance scores such that $g_j = 2^{y_j}-1$.
We can then rewrite $\mathrm{DCG}(\vy, \hat\pi)$ as:

 \begin{align}
 \label{eq:DCG_sim}
\mathrm{DCG}(\vy, \hat\pi)&= \sum_{j=1}^{L} \frac{g_{\hat\pi_j}}{\log_2(1+j)}
=\sum_{j=1}^{L}\frac{[P_{\hat\pi} \vg]_j}{\log_{2}(1+j)}.
 \end{align}

To obtain the DCG@$k$ objective, one can replace $L$ with $k$ in the sum.
We omit the suffix @$k$ in the following, assuming that $k$ has been defined, potentially equal to $L$ which would yield the full metric.

Let $\widehat{P}_{\sort{}(\mathbf{s})}(\tau)$ denote a relaxation to the permutation matrix $P_{\sort{}(\mathbf{s})}$ that can be used for differentiable sorting of an input score vector $\mathbf{s}$, for some temperature parameter $\tau > 0$ such that the true matrix is recovered as $\tau \to 0^+$.
Since $\hat\pi= \mathrm{sort}(\hat{\vy})$, we can obtain a differentiable relaxation to $\mathrm{DCG}(\vy,\hat\pi)$:
 
  \begin{align}
 \label{eq:DCG_soft}
\widehat{\mathrm{DCG}}(\vy, \hat{\vy}, \tau) &=\sum_{j=1}^{k}\frac{[\widehat{P}_{\sort{}(\hat{\vy})} (\tau)\vg]_j}{\log_{2}(1+j)}.
 \end{align}

\begin{figure}
    \centering
    \includegraphics[width=0.50\textwidth]{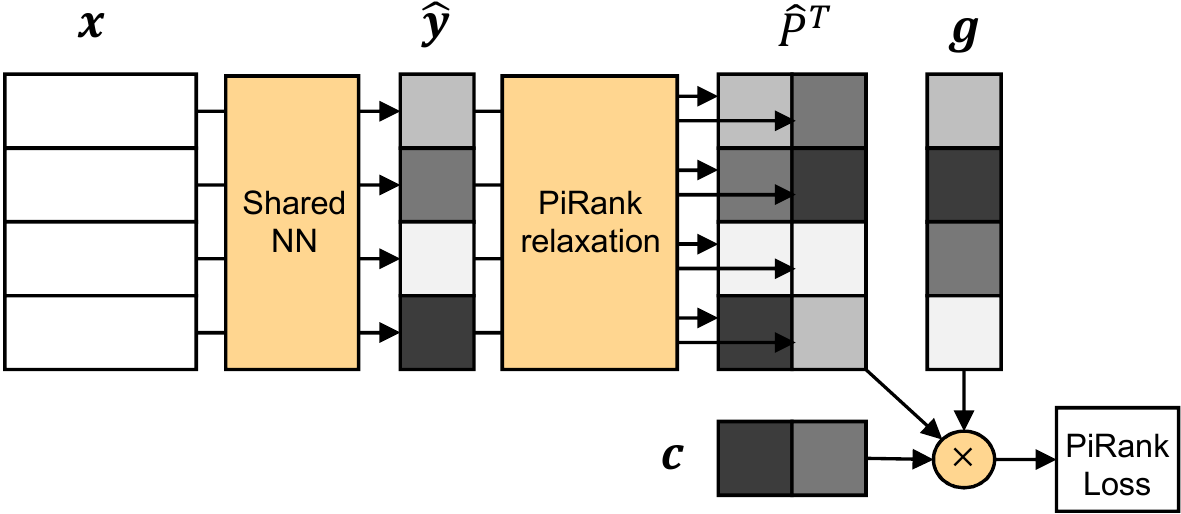}
    \caption{Architecture for the computation of the PiRank relaxed NDCG@$k$ loss for $L = 4$ and $k=2$. Square cells represent scalars with darker shades indicating higher values. The fourth item has currently the highest score as given by the neural network but the second item has the highest relevance.
    The vector $\bm{c}$, with components $c_j = 1 / \log(1 + j)$, discounts gains $\bm{g}$ based on rankings.}
    \label{fig:loss}
\end{figure}

 Substituting this in the expression for NDCG in Eq.~\ref{eq:NDCG}, we obtain the following relaxation for NDCG:
\begin{align}\label{eq:neuralrank}
\widehat{\mathrm{NDCG}}(\vy, \hat{\vy}, \tau)=\frac{\widehat{\mathrm{DCG}}(\vy, \hat{\vy}, \tau)}{\mathrm{DCG}(\vy, \pi^\ast)}
\end{align}
where the normalization in the denominator does not depend on $\theta$ and can be computed exactly via regular sorting.
Finally, we define the \name{} surrogate loss for NDCG as follows:
\begin{align}
\label{eq:neuralrank_loss}
    \ell_{\mathrm{PiRank-NDCG}} = 1 - \widehat{\mathrm{NDCG}}(\vy, \hat{\vy}, \tau)
\end{align}
which is bounded between 0 and 1 as is NDCG, and whose difference with the actual $(1 - \text{NDCG})$ gets negligible as $\tau \to 0^+$.
Figure \ref{fig:loss} illustrates the model architecture for the above objective. 
Similarly, we can derive a surrogate loss for the ARP metric in Eq.~\ref{eq:ranked_precision} as:
\begin{align}\label{eq:neuralrank_rp}
\hat{\ell}_{\mathrm{PiRank-ARP}}(\boldsymbol{y}, \hat{\vy}, \tau)=\frac{\sum_{j=1}^{k}[\widehat{P}_{\sort{}(\hat{\vy})} (\tau)\vy]_j j}{\sum_{j=1}^{k} y_{j}}.
\end{align}

\subsection{Example: Differentiability via NeuralSort}
\label{sec:algo}

Typically, relaxations to permutation matrices consider the Birkhoff polytope of doubly stochastic matrices. 
A doubly-stochastic matrix is a square matrix with entries in $[0,1]$ where every row and column sum to $1$.
In contrast, NeuralSort\cite{grover2019stochastic} is a recently proposed relaxation of permutation matrices in the space of \textit{unimodal} row-stochastic matrices. 
A unimodal matrix is a square matrix with entries in $[0,1]$ such that the entries in every row sum to 1 (i.e. row-stochastic), but additionally enforce the constraint that the maximizing entry in every row should have a unique column index.
See Figure~\ref{fig:ds_uni} for an example of each type.
Note that a unimodal matrix is not necessarily doubly-stochastic and vice versa.
Permutation matrices are both doubly-stochastic and unimodal.

\begin{figure}
\small
\centering
\begin{subfigure}[b]{.4\columnwidth}
\centering
\begin{tikzpicture}
\matrix[matrix of math nodes,
        left delimiter=(,
        right delimiter=),
        nodes in empty cells] (m)
{
0       & \textbf{0.9} & 0.1  \\
\textbf{0.5} & 0.01  &  0.49     \\
\textbf{0.5} & 0.09  & 0.41     \\
};
\end{tikzpicture}
 \end{subfigure}
\begin{subfigure}[b]{.4\columnwidth}
\centering
\begin{tikzpicture}
\matrix[matrix of math nodes,
        left delimiter=(,
        right delimiter=), 
        nodes in empty cells] (m)
{
\textbf{0.8} & 0.2  &  0   \\
0.2       & 0.3 & \textbf{0.5}  \\
0.25 & \textbf{0.6}   & 0.15    \\
};
\end{tikzpicture}
 \end{subfigure}
 \caption{Doubly-stochastic (left) vs. unimodal (right) matrices. Maximum entry in every row in \textbf{bold}. Unlike unimodal matrices, two different items can have the same assignment of most-likely ranks (column indices) for doubly-stochastic matrix relaxations.
 \vspace{-0.25in}
 } \label{fig:ds_uni}

\end{figure}

 In NeuralSort~\cite{grover2019stochastic}, a unimodal relaxation of the permutation matrix $P_{\sort{}(\hat \vy)}$ can be defined as follows.
Let $A_{\hat\vy}$ denote the matrix of absolute pairwise score differences with $i, j$-th entry given as $[A_{\hat\vy}]_{ij} = \vert \hat y_i - \hat y_j \vert$.
Then, the $i$-th row of the relaxed permutation matrix is:
\begin{align}\label{eq:relaxed_perm}
\widehat{P}^{\,(NS)}_{\sort{}(\hat\vy)}(\tau)_{i,\cdot} = \mathrm{softmax} 	\left[((L+1-2i)\hat\vy - A_{\hat\vy} \mathds{1})/\tau\right]
\end{align}
where $\mathds{1}$ is the vector with all components equal to 1.
Its unimodal property makes it particularly well-suited to extracting top-$k$ items because, as seen in Figure~\ref{fig:ds_uni}, taking the maximizing elements of the first $k$ rows yields exactly $k$ items but may yield less in the case of a doubly-stochastic relaxation.
However, the complexity to obtain the top-$k$ rows in this formulation, even for $k$ as low as 1, is quadratic in $L$
as the full computation of $A_{\hat\vy} \mathds{1}$ is required for the softmax operation in Eq.~\ref{eq:relaxed_perm}.
This is prohibitive when $L \gg k$, a common scenario, and motivates the introduction of a new relaxation with a more favorable complexity for top-$k$ ranking.

\subsection{Scaling via Divide-And-Conquer} 

Our \name{} losses only require the first $k$ rows of the relaxed permutation matrix $\widehat{P}_{\sort{}(\hat{\vy})}$.
This is specific to the LTR setting in which only the top-ranked items are of interest, in contrast to the full sorting problem that requires the full matrix.
In \name{}, we leverage this insight to construct a divide-and-conquer variant of differentiable sorting relaxations such as NeuralSort to reduce the complexity of the metric computation.
 Our proposed construction can be viewed as a relaxed and truncated multi-way merge sort algorithm with differentiable sorting relaxations as building blocks.
 In the following discussion, we use NeuralSort as our running example while noting that the analysis extends more generally to other differentiable relaxations as well.

\paragraph{Data Structure Construction.}
Let $L = b_1 b_2 \cdots b_d$ be a factorization of the list size $L$ into $d$ positive integers.
Using this factorization, we will construct a tree of depth $d$ with branching factor $b_j$ at height $j$. 
Next, we split the $L$-dimensional score vector $\hat \vy$ into its $L$ constituent scalar values.
We set these values as the leaves of the tree.
See Figure~\ref{fig:tree} for an example.
At every other level of the tree, we will merge values from the level below into equi-sized lists.
Let $\{k_j\}_{j=0}^d$ be sizes for the intermediate results at level $j$, such that $k_0=1$ (leaves) and $\min(k, k_{j-1} b_j) \leq k_j \leq k_{j-1} b_j$ for $j\geq 1$ (explained below).
Then, in an iterative manner for levels $j=1, \dots, d$,
the value of a node at height $j$ are the top-$k_j$ scores given by the application of the NeuralSort operator on the concatenation of the values of its children.
With $k_d=k$, the root value thus obtained is a relaxation of the top-$k$ scores in $\hat \vy$.
The top-$k$ rows of the relaxed permutation matrix $\widehat{P}_{\sort{}(\hat{\vy})}$ yielding these scores are constructed by compounding the operations at each iteration.

\begin{figure}
    \centering
    \includegraphics[width=0.45\textwidth]{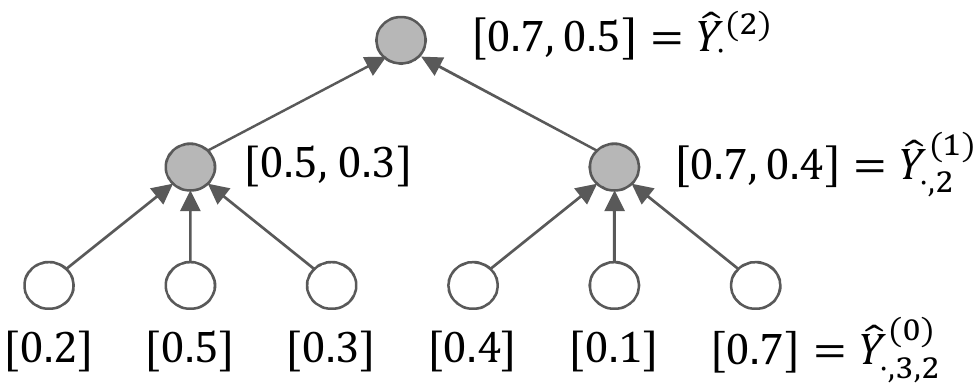}
    \caption{Divide-and-conquer strategy for $L= 6 =3\cdot 2$, $k = 2$ and $\hat\vy^T = (0.2, 0.5, 0.3, 0.4, 0.1, 0.7)$.
    The scores are merged in groups of size $b_1 = 3$ and the respective top $k_1=2$ scores are kept, then the $b_2=2$ outputs are merged to obtain the final top $k_2=k=2$ scores.
    The effect of relaxation is not shown for readability.
    At non-zero temperature, the values at non-terminal nodes would be linear combination of the scores.
    \vspace{-0.15in}
    }
    \label{fig:tree}
    
\end{figure}

\paragraph{Computational Complexity.} The intuition behind the favorable scaling is as follows.
At step $j$, NeuralSort is applied on blocks of size $k_{j-1} b_j$ as it merges $b_j$ sub-blocks of size $k_{j-1}$.
Obtaining the full sorted list of scores would require to keep all intermediate scores during the process, i.e., $k_j^{\textrm{(max)}} = b_1 \cdots b_j = k_{j-1} b_j$
for $j \geq 1$.
In the last step, the NeuralSort operator is applied on a list of size $k_{d-1} b_d$, equal to $L$ in this case, so the overall complexity would be at least quadratic in $L$ as explained previously.
However, since only the top-$k$ scores are desired, intermediate outputs can be truncated if larger than $k$.
Full truncation corresponds to $k_j^{\textrm{(min)}} = \min(k, k_{j-1} b_j)$.
Any choice $k_j^{\textrm{(min)}} \leq k_j \leq k_j^{\textrm{(max)}}$ is acceptable to recover the top-$k$ scores, with larger values allowing more information to flow at the expense of a higher complexity.
Choosing $b_j \approx L^{1/d}$ and $k_j$ minimal,
the list sizes $b_j k_{j-1}$ on which NeuralSort is applied at each step can thus be of the order of $L^{1/d}k$, much smaller than $L$ in the $d > 1$ and $k \ll L$ scenario.

Formally, 
let $\tau_1, \tau_2, \dots, \tau_d$ be the relaxation temperatures at each height, with $\tau_d = \tau$ and $\tau_j \leq \tau_{j+1}$ for $j \in \{1, \dots, d - 1\}$.
Define the tensor $\hat Y^{(0)}$ by reshaping $\hat{\vy}$ to shape $(k_0, b_1, b_2, \dots, b_d)$, yielding components
\begin{align}
    \hat Y^{(0)}_{1, i_1, i_2, \dots, i_d} = \hat y_{1+\sum_{j=1}^d (i_j-1) \prod_{l=1}^{j-1} b_l},
\end{align}
with $i_j \in \{1, \dots, b_j\}$ and the first index is always 1 as $k_0 = 1$.
With the tree representation, the first tensor index is the position in the node value vector and the rest of the indices identify the node by the index of each branching starting from the root. 
For $j \in \{1, 2, \dots, d\}$, recursively define the tensors $\hat Q^{(j)}$, $\hat Y^{(j)}$ and $\hat P^{(j)}$ of respective shapes $(k_j, k_{j-1}, b_j, \dots, b_d)$, $(k_j, b_{j+1}, \dots, b_d)$ and $(k_j, b_1, \dots, b_d)$ with components
\begin{equation}
\begin{aligned}
    \hat Q^{(j)}_{l, m, i_j, \dots, i_d} &= \operatorname{softmax}\Bigl[\Bigl((k_{j-1} b_j + 1 - 2l) \hat Y^{(j-1)}_{m, i_j, \dots, i_d} \bigr.\bigr. 
    - \Bigl.\Bigl.
    \sum_{p=1}^{k_{j-1}} &\sum_{q=1}^{b_j} \left|\hat Y^{(j-1)}_{m, i_j, i_{j+1}, \dots, i_d} - \hat Y^{(j-1)}_{p, q, i_{j+1}, \dots, i_d} \right| \Bigr) / \tau_j \Bigr],
 \end{aligned}
\end{equation}
\begin{equation}
\begin{aligned}
    \hat Y^{(j)}_{l, i_{j+1}, \dots, i_d} &= \sum_{p=1}^{k_{j-1}} \sum_{q=1}^{b_j} \hat Q^{(j)}_{l, p, q, i_{j+1}, \dots, i_d} \hat Y^{(j-1)}_{p, q, i_{j+1}, \dots, i_d},
 \end{aligned}
\end{equation}
\begin{equation}
\begin{aligned}
    \hat P^{(j)}_{l, i_1, \dots, i_d} &= \sum_{m=1}^{k_{j-1}} \hat Q^{(j)}_{l, m, i_j, \dots, i_d} \hat P^{(j-1)}_{m, i_1, i_2, \dots, i_d},
\end{aligned}
\end{equation}
with $\hat P^{(0)} = 1$.
Intuitively, $\hat Y^{(j)}$ holds the relaxed top-$k_j$ scores at height $j$ and $\hat Y^{(d)}$ is the desired top-$k$ score vector.
The interpretation of the indices in the tree structure is as for $\hat Y^{(0)}$, illustrated in Figure~\ref{fig:tree}.
More importantly, we keep track of the relaxed sorting operation that yielded this output.
$\hat Q^{(j)}$ is the relaxed permutation matrix obtained by applying NeuralSort in Eq.~\ref{eq:relaxed_perm} to $\hat Y^{(j)}$, while $\hat P^{(j)}$ compounds the relaxed permutation matrices obtained so far so it always maps from the initial list size.
 Finally, define the $k \times L$ matrix $\hat P$ by reshaping the tensor $\hat P^{(d)}$, yielding components
\begin{align}
    \hat P_{l, 1 + \sum_{j=1}^d (i_j - 1) \prod_{l=1}^{j-1} b_l} = \hat P^{(d)}_{l, i_1, \dots, i_d},
\end{align}
for $i_j \in \{1, \dots, b_j\}$.
The $k$ rows of $\hat P$ are used as the top-$k$ rows of the relaxed sorting operator $\widehat{P}_{\sort{}(\hat \vy)}(\tau)$.
This approach is equivalent to NeuralSort, yielding Eq.~\ref{eq:relaxed_perm} for $d=1$.
Proof of convergence for $\tau \to 0^+$ of this relaxation in the general case $d\geq 1$ is presented in Appendix~\ref{sec:proof}.

In the simple case where $L=b^d$ and we set $b_j = b$, $k_j=\min(k, b^j)$ \s{this doesn't seem consistent with the previous definition of kj?}
\bc{maybe clearer now}
for all $j\in\{1, \dots, d\}$, the complexity to compute $\hat P$ and thus the \name{} losses is then $O(L^{1+1/d} + (d-1)k^2L)$, which scales favorably in $L$ if $d>1$ and $k = O(1)$.
In the general case, the score list can be padded, e.g. to the power of 2 following $L$, such that the previous complexity holds for $b=2$ and $d=\lceil \log_2 L \rceil$, but other factorizations may yield lower complexity depending on $L$.

\section{Experiments}

\label{sec:experiments}

We present two sets of experiments in this section: (a) benchmark evaluation comparing \name{} with other ranking based approaches on publicly available large-scale benchmark LTR datasets, and (b) ablation experiments for the design choices in \name{}.

\subsection{Benchmark Evaluation via TF-Ranking}
\label{sec:benchmark}

\begin{table*}[t]
\centering
\scriptsize
\begin{tabular}{|l|llllllll|}
\hline
Loss / Metric & \multicolumn{1}{c}{{OPA}} & \multicolumn{1}{c}{{ARP}} & \multicolumn{1}{c}{{MRR}} & \multicolumn{1}{c}{{NDCG@1}} & \multicolumn{1}{c}{{NDCG@3}} & \multicolumn{1}{c}{{NDCG@5}} & \multicolumn{1}{c}{{NDCG@10}} & \multicolumn{1}{c|}{{NDCG@15}} \\ \hline
RankNet                                  & 0.611494                & 46.746979               & 0.786148                & 0.331595                   & 0.336593                   & 0.346928                   & 0.375944                    & 0.398582                    \\
LambdaRank                               & 0.618954                & 46.174503               & 0.798169                & 0.392150                   & 0.396045                   & 0.404275                   & 0.425611                    & 0.444942                    \\
Softmax                                  & 0.612626                & 46.557617               & 0.761750                & 0.331527                   & 0.338999                   & 0.353011                   & 0.381717                    & 0.405312                    \\
Approx. NDCG                             & 0.630616                & 45.461678               & \textbf{0.814873}       & \textbf{0.423497}          & 0.409272                   & 0.414501                   & 0.434463                    & 0.453627                    \\
\hline
NeuralSort                               & \textbf{0.635468}       & \textbf{44.966999}      & 0.779865                & 0.373344                   & 0.386647                   & 0.402052                   & 0.430580                    & 0.452863                    \\
PiRank-NDCG                              & 0.629763                & 45.394020               & \textbf{0.813016}       & \textbf{0.425006}          & \textbf{0.420569}          & \textbf{0.426034}          & \textbf{0.446428}           & \textbf{0.465309}           \\
\hline
\end{tabular}
\\[1ex]
\scriptsize
\begin{tabular}{|l|llllllll|}
\hline
Loss / Metric & \multicolumn{1}{c}{{OPA}} & \multicolumn{1}{c}{{ARP}} & \multicolumn{1}{c}{{MRR}} & \multicolumn{1}{c}{{NDCG@1}} & \multicolumn{1}{c}{{NDCG@3}} & \multicolumn{1}{c}{{NDCG@5}} & \multicolumn{1}{c}{{NDCG@10}} & \multicolumn{1}{c|}{{NDCG@15}} \\ \hline
RankNet                                  & 0.625170                & 10.514806               & 0.889641                & 0.599739                   & 0.622155                   & 0.650367                   & 0.704332                    & 0.733015                    \\
LambdaRank                               & 0.636126                & 10.451870               & 0.896020                & 0.633181                   & 0.652220                   & 0.676243                   & 0.723489                    & 0.749001                    \\
Softmax                                  & 0.627313                & 10.472106               & 0.886976                & 0.588957                   & 0.618409                   & 0.648358                   & 0.703746                    & 0.732624                    \\
Approx. NDCG                             & 0.648793                & 10.277598               & \textbf{0.903356}       & \textbf{0.668700}          & \textbf{0.670107}          & 0.690353                   & 0.735641                    & 0.760539                    \\
\hline
NeuralSort                               & 0.648416                & 10.320462               & 0.898246                & 0.640018                   & 0.655990                   & 0.681075                   & 0.729225                    & 0.754867                    \\
PiRank-NDCG                              & \textbf{0.654255}       & \textbf{10.250481}      & \textbf{0.902088}       & 0.661661                   & \textbf{0.672390}          & \textbf{0.693825}          & \textbf{0.738479}           & \textbf{0.763864} \\
\hline
\end{tabular}
\caption{Benchmark evaluation on (upper) MSLR-WEB30K and (lower) Yahoo! C14 test sets.  In bold, the best performing method and all other methods not significantly worse.}
\label{tab:benchmarks}
\end{table*}

\paragraph{Datasets.}
To empirically test PiRank, we consider two of the largest open-source benchmarks for LTR:  the MSLR-WEB30K\footnote{\label{note:mslr}https://www.microsoft.com/en-us/research/project/mslr/} and the Yahoo! LTR dataset C14\footnote{https://webscope.sandbox.yahoo.com}.
Both datasets have relevance scores on a 5-point scale of 0 to 4, with 0 denoting complete irrelevance and 4 denoting perfect relevance.
We give extensive details on the datasets and experimental protocol in Appendix~\ref{sec:datasets}.

\paragraph{Baselines.}
We focus on neural network-based approaches and use the open-source TensorFlow Ranking (TFR) framework~\cite{pasumarthi2019tf}.
TFR consists of high-quality GPU-friendly implementations of several LTR approaches, common evaluation metrics, and standard data loading formats.
We compare our proposed loss, \name{}-NCDG, with the following baselines provided by TensorFlow Ranking: Approximate NDCG Loss~\cite{qin2010general}, Pairwise Logistic Loss (RankNet, Eq.~\ref{eq:ranknet}), Pairwise Logistic Loss with lambda-NDCG weights (LambdaRank, Eq.~\ref{eq:lambdarank}), and the Softmax Loss.
We also include NeuralSort, whose loss is the cross-entropy of the predicted permutation matrix.
Of these methods, the Pairwise Logistic Loss (RankNet) is a pairwise approach while the others are listwise.
While our scope is on differentiable ranking surrogate losses for training neural networks, other methods such as the tree-based LambdaMART \cite{Burges2010} could potentially yield better results.

\paragraph{Setup.} All approaches use the same 3-layer fully connected network architecture with ReLU activations to compute the scores $\hat{\boldsymbol{y}}$ for all (query, item) pairs, trained on 100,000 iterations.
The maximum list size for each group of items to score and rank is fixed to 200, for both training and testing.
Further experimental details are deferred to Appendix~\ref{sec:exp_details}.
We evaluate Ordered Pair Accuracy (OPA), Average Relevance Position (ARP), Mean Reciprocal Rank (MRR), and NDCG@$k$ with $k \in \{1, 3, 5, 10, 15\}$.
We determine significance similarly to~\cite{Reddi2021RankDistilKD}.
For each metric and dataset, the best performing method is determined, then a one-sided paired t-test at a 95\% significance level is performed on query-level metrics on the test set to compare the best method with every other method.

\paragraph{Results.}

Our results are shown in Table~\ref{tab:benchmarks}.
Overall, \name{} shows similar or better performance than the baselines on 13 out of 16 metrics.
It significantly outperform all baselines on NDCG@$k$ for $k \geq 5$ while Approx. NDCG is competitive on NDCG@$k$ for $k \leq 3$ and MRR.
\name{} is also significantly superior on OPA and ARP metrics on Yahoo! C14 while NeuralSort is superior for MSLR-WEB30K.

\subsection{Ablation Experiments}
\label{sec:ablation}

\begin{figure}[t]
    \centering
    \includegraphics[width=0.60\textwidth,trim={0.2cm 0.4cm 0.2cm 0.3cm}, clip]{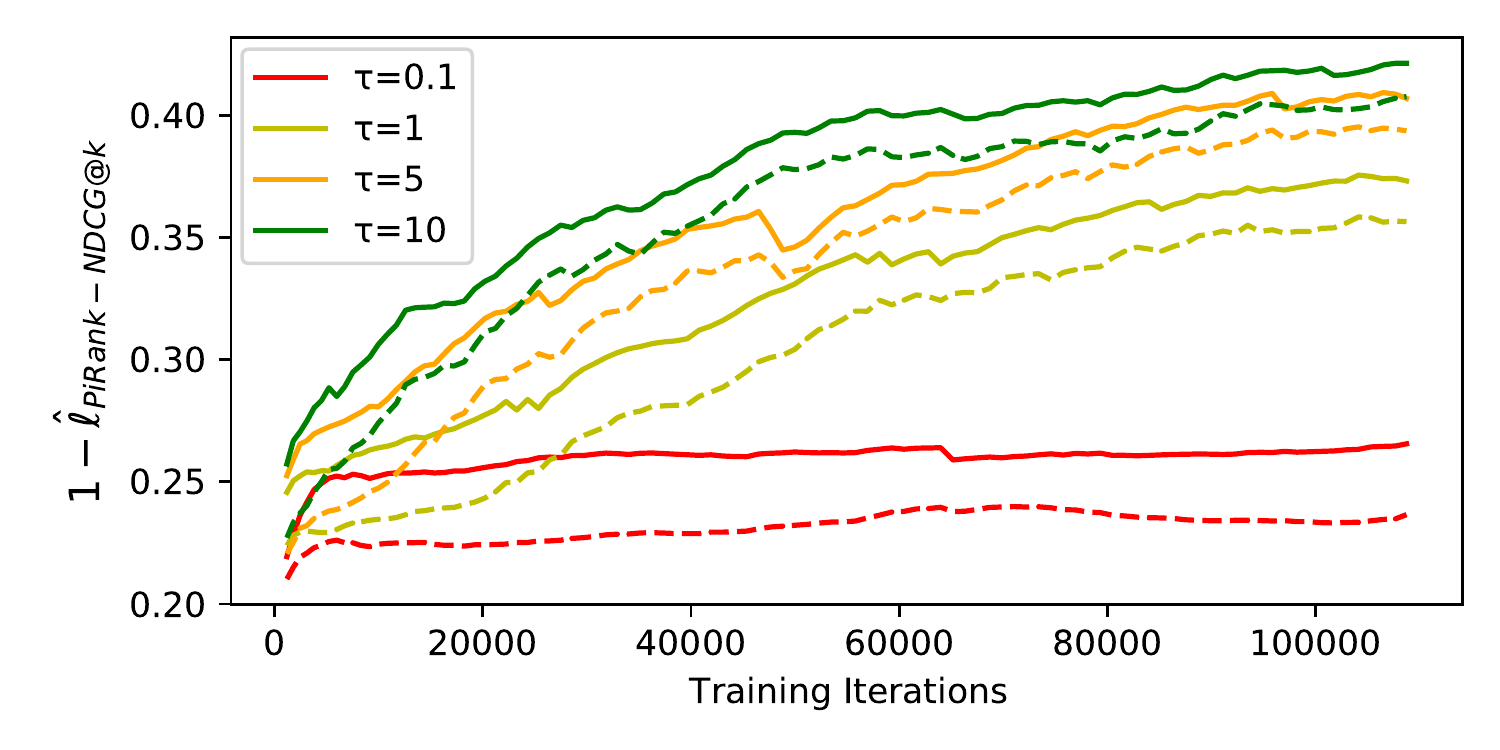}
    \caption{$1 -\hat{\ell}_\mathrm{\name{}-NDCG@k}$ ($k=10$, full lines) for different values of the temperature parameter $\tau$, with the corresponding value of the hard metric NDCG@10 (dashed lines), at validation.
    }
    \label{fig:temperature}
\end{figure}

\paragraph{Temperature.}

The temperature hyperparameter $\tau$ is used in \name{} to control the degree of relaxation.
We experiment on several values ($\tau \in \{0.1, 1, 5, 10\}$)
using the MSLR-WEB30K dataset and the experimental settings for ablation provided in Appendix~\ref{sec:exp_details}.
Figure~\ref{fig:temperature} demonstrates the importance of correctly tuning $\tau$.
High values ($\tau > 1$) speed up training, especially in the early regime, while low values induce large gradient norms which are unsuitable for training and lead to the loss stalling or even diverging.
Another observation is that the relaxed metric $1 - \hat{\ell}_{\mathrm{\name{}-NDCG@}k}$ closely follows the value of the downstream metric $\text{NDCG@}k$ as expected.
We further experimented with a very high temperature value and an exponentially decreasing annealing schedule in Appendix~\ref{sec:size_exps}.

\begin{figure}[t]
    \centering
    \includegraphics[width=0.55\textwidth,trim={0.2cm 0.2cm 0.2cm 0.2cm}, clip]{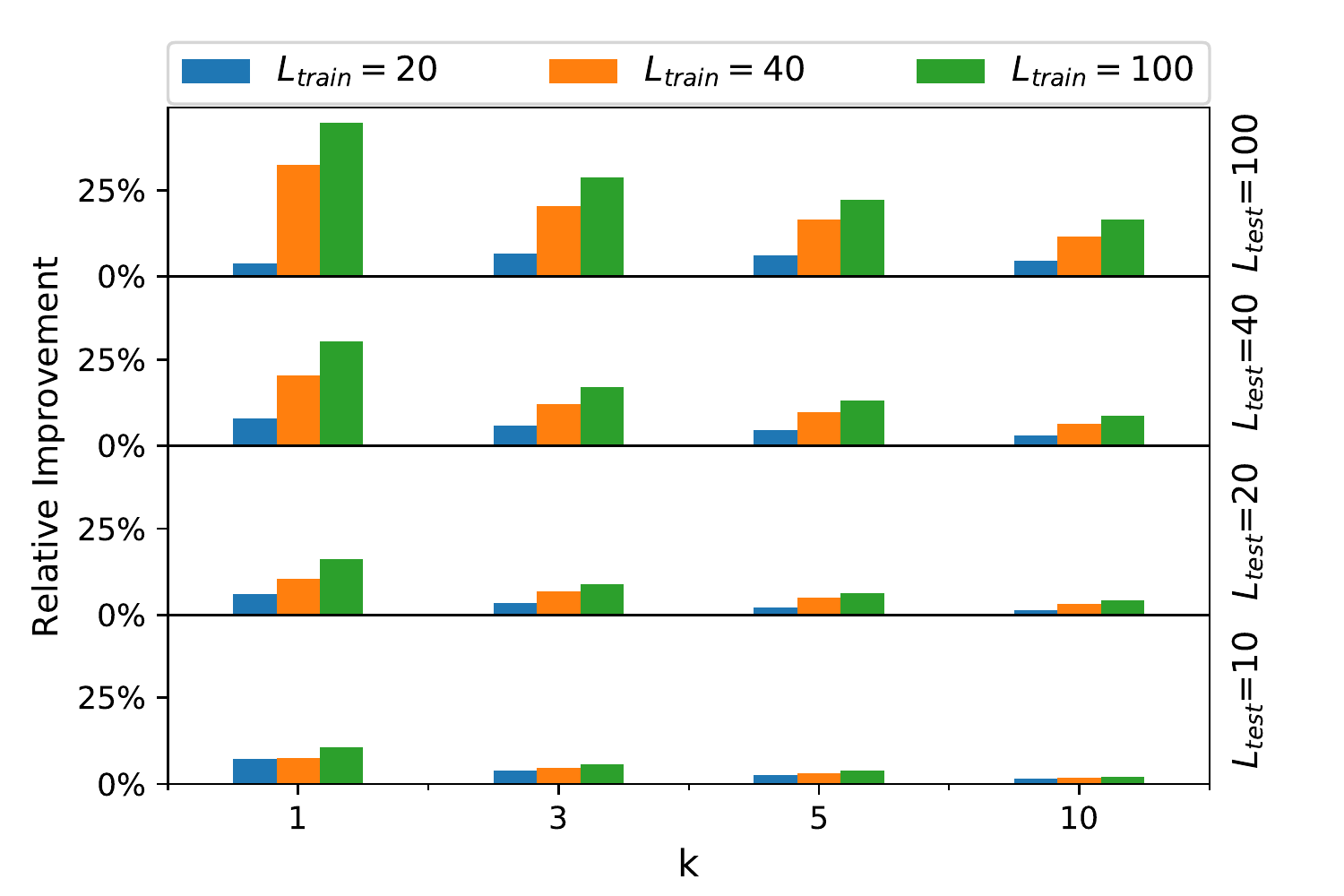}
    \caption{Relative improvement of NDCG@$k$ on different values of $L_{test}$, for different $L_{train}$ values vs. a baseline of $L_{train}=10$.}
    \label{fig:sizes}
\end{figure}

\paragraph{Training List Size.}

The training list size parameter $L_{train}$ determines the number of items to rank for a query during training.
The setting is the same as for the temperature ablation experiment, but with training list sizes $L_{train} \in \{10, 20, 40, 100\}$ which we then evaluate on testing list sizes in the same range of values $L_{test} \in \{10, 20, 40, 100\}$. The dataset is again MSLR-WEB30K. 
Figure~\ref{fig:sizes} exposes four patterns for NDCG@$k$.
First, for a fixed $L_{test}$ and $k$, a larger $L_{train}$ is always better.
Second, for a fixed $L_{test}$, we observe diminishing returns along $k$, as relative improvements decrease for all $L_{train}$.
This observation is confounded by NDCG@$k$ values growing larger with $k$, but the metric is always able to distinguish between ranking functions \cite{Wang2013}.
Third, for a fixed $k$, our returns along $L_{test}$ increase with $L_{train}$ (except for $L_{train}=20$ and $k=1$). This means that the need for a larger $L_{train}$ is more pronounced for larger values of $L_{test}$.
Fourth and last, the returns increase most dramatically with $L_{train}$ when $L_{test} \gg k$ (top left), a common industrial setting. 
Values for NDCG@$k$, MRR, OPA, ARP are provided in Appendix~\ref{sec:size_exps}. For MRR, using a larger $L_{train}$ is always beneficial regardless of $L_{test}$, but not always for OPA and ARP. 

\label{sec:depthablation}

 \begin{figure}[t]
    \centering
    \includegraphics[width=0.55\textwidth,trim={0.3cm 0.4cm 0.2cm 0}, clip]{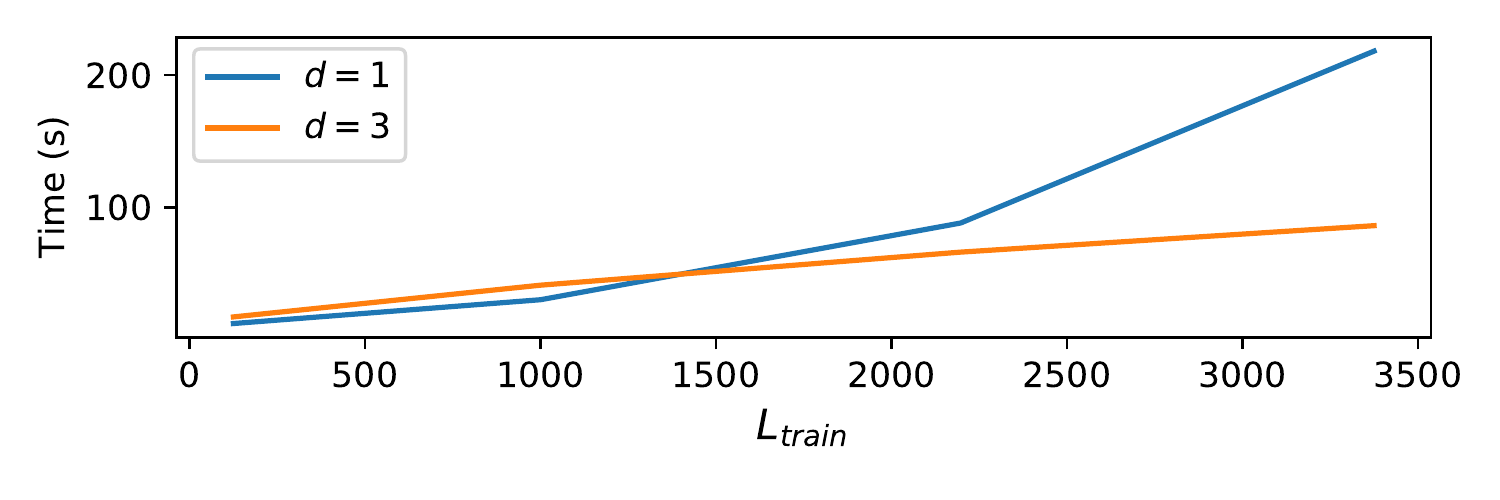}
    \caption{Wall-clock time for 100 training steps, each corresponding to 16 queries, for different $L_{train}$ and maximal depth $d$. 
    We use $k = 1$ and $L_{train}=5^3, 10^3, 13^3$ and $15^3$ such that $L_{train}=b^3$ for $d=3$ and $k_j=1$. Variation across runs is too small ($\sim$1s) and omitted for brevity.}
    \label{fig:runtime}
\end{figure}

\paragraph{Depth.}
A main advantage of the \name{} loss is how it can scale to very large training list sizes $L_{train}$.
This setting is difficult to come across with traditional LTR datasets, which are manually annotated, but occurs frequently in practice.
One example is when the relevance labels are obtained from implicit signals such as clicks or purchases in recommendation systems.
In this case, an LTR model is used to re-rank a list of candidates generated by another, simpler, model choosing among all possible items those potentially relevant to a query or context.
An LTR model capable of handling very large lists can reduce the impact of errors made by the simpler candidate generation step, moving to the top an item lowly ranked at first that would have been cut off from a smaller list.
To test the complexity shown in Section~\ref{sec:algo} in extreme conditions, we create a synthetic dataset as described in Appendix~\ref{sec:synthetic}.
Figure~\ref{fig:runtime} shows how the training time for depth $d=3$ scales much more favorably than for $d=1$, following their respective time complexities of $O(L^{1+1/3})$ and $O(L^2)$.

\section{Summary and Limitations}

\label{sec:conclusion}

We proposed \name{}, a novel class of surrogate loss functions for Learning-To-Rank (LTR) that leverages a continuous, temperature-controlled relaxation to the sorting operator \cite{grover2019stochastic} as a building block. 
This allows us to recover exact expressions of the commonly used non-differentiable ranking metrics in the limit of zero temperature, which we proved in particular for the NDCG metric.
Crucially, we proposed a construction inspired by the merge-sort algorithm that permits \name{} to scale to very large lists.

In our experiments on the largest publicly available LTR datasets, we observed that \name{} has superior or similar performance with competing methods on the MSLR-WEB30K and Yahoo! C14 benchmarks on 13/16 ranking metrics and their variants. 

As future work, we would like to explore other recent relaxations of the sorting operator~\cite{blondel2020fast,prillo2020softsort,xie2020differentiable} as a building block for the PiRank framework.
Further, as ranking is a core component of modern day technology and influences everyday decision making pipelines involving vulnerable populations, care needs to be taken that our proposed systems are extended to account for biases and fairness criteria when deployed in real world settings.

\ifanonymous
\else
\section{Acknowledgements}

Robin Swezey and Bruno Charron were supported by Rakuten, Inc. \\
Stefano Ermon is supported in part from NSF (\#1651565, \#1522054, \#1733686), ONR (N00014- 19-1-2145), AFOSR (FA9550-19-1-0024) and Bloomberg.
\fi

\bibliographystyle{unsrt}
{\footnotesize\bibliography{refs}}

\newpage
\pagebreak

\newpage
\pagebreak
\appendix 
\section*{Appendices}

\section{Further Related Works}
\label{sec:further_related}



\paragraph{Pairwise approaches.} Closely related to RankNet, pairwise approaches such as Sortnet~\cite{rigutini2011sortnet} and SmoothRank~\cite{chapelle2010gradient} casts sorting of $n$ elements as performing $n^2$ pairwise comparisons, and try to approximate the pairwise comparison operator for sorting. We consider a more direct relaxation with attractive properties for rankings that we describe in Section~\ref{sec:technical}.

\paragraph{Listwise approaches.} include ListNet~\cite{Cao2015} and ListMLE~\cite{xiong2017end}, which define surrogate losses that take into consideration the full predicted rank ordering while being agnostic to the downstream ranking metrics. 
ListNet for instance considers the predicted scores as parameters for the Plackett-Luce distribution~\cite{plackett1975analysis,luce2012individual} and learns these scores via maximum likelihood estimation.

\section{Proof of Convergence}
\label{sec:proof}

Used in a PiRank surrogate loss of Section~\ref{sec:relaxndcg}, the relaxation presented in Section~\ref{sec:algo} recovers the downstream metric by lowering the temperature as formalized in the result below for NDCG.
\begin{proposition}\label{thm:tmp}
If we assume that the entries of $\hat{\vy}$ are drawn independently from a distribution that is absolutely continuous w.r.t. the Lebesgue measure in $\mathbb{R}$, then the following convergence holds almost surely:
\begin{align}
\lim_{\tau\to 0^{+}} \hat{\ell}_{\mathrm{PiRank-NDCG}}(\vy, \hat{\vy}, \tau) = 1-\mathrm{NDCG}(\vy, \hat\pi)
\end{align}
where $\hat\pi = \sort{}(\hat{\vy})$.
\end{proposition}

\begin{proof}
In the $d>1$ case, the limit is interpreted as 
\begin{align}
    \lim_{\tau\to 0^{+}} = \lim_{\tau_d\to 0^{+}} \lim_{\tau_{d-1}\to 0^{+}} \dots \lim_{\tau_1\to 0^{+}}
\end{align}
given the increasing ordering of the temperatures by height and the constraint $\tau_d = \tau$.

We first sketch a proof by induction on the height $j$ that, under the same assumptions as the proposition, for all $i_{j+1}, \cdots, i_d$, the $k'_j$-dimensional vector
\begin{align}\label{eq:proof_y}
    \vy^{(j)}_{i_{j+1}, \cdots, i_d} \equiv \lim_{\tau_j\to 0^{+}} \hat Y^{(j)}_{:k'_j, i_{j+1}, \cdots, i_d}
\end{align}
with $k'_j = \min(k, k_j)$ and $:l$ the top-$l$ rows extraction,
contains the top-$k'_j$ scores in $\hat Y^{(0)}_{\cdot, \cdots, \cdot, i_{j+1}, \cdots, i_d}$ in descending order and the $k'_j \times L_j$ matrix
\begin{align}
    P^{(j)}_{i_{j+1}, \cdots, i_d} \equiv \lim_{\tau_j\to 0^{+}} \hat P^{(j)}_{:k'_j, i_{j+1}, \cdots, i_d}
\end{align}
with $L_j = b_1 \cdots b_j$ is the row-truncated permutation matrix realizing the ordering,
\begin{align}\label{eq:proof_py}
    \vy^{(j)}_{i_{j+1}, \cdots, i_d} = P^{(j)}_{i_{j+1}, \cdots, i_d} \hat Y^{(0)}_{\cdot, \cdots, \cdot, i_{j+1}, \cdots, i_d}
\end{align}
where reshaping as necessary is implicit in the above two equations.

For $j=0$, this is trivial as $P^{(0)} = 1$ and by convention $b_0 = k_0 = 1$.

Assuming the above is true for a height $j-1$, the top-$k'_j$ scores in $\hat Y^{(0)}_{\cdot, \cdots, \cdot, i_{j+1}, \cdots, i_d}$ are included in the concatenation of the vectors $\hat Y^{(j-1)}_{\cdot, i_j, \cdots, i_d}$ for $i_j \in \{1, \dots, b_j\}$ in the $\tau_{j-1}\to 0^+$ limit from the assumption (no limit for $j=1$).
$\hat Q^{(j)}_{\cdot, \cdot, \cdot, i_{j+1}, \cdots, i_d}$ is then the NeuralSort relaxed permutation matrix for these concatened vector.
From Theorem 1 of \cite{grover2019stochastic}, we know that in the $\tau_j \to 0^+$ limit, this matrix will converge to the sorting permutation matrix.
In this limit, $\hat Y^{(j)}_{\cdot, i_{j+1}, \cdots, i_d}$ is then sorted version of the concatened vector, so that in particular its top-$k'_j$ elements are the sorted top-$k'_j$ elements of the concatenated vector, proving the claim on $\vy^{(j)}_{i_{j+1}, \cdots, i_d}$.
Further, the claim on $P^{(j)}$ directly derives from the previous observation on the limit of $\hat Q^{(j)}_{\cdot, \cdot, \cdot, i_{j+1}, \cdots, i_d}$ and the fact that a product of permutation matrices which is the matrix of the product of the permutations.
This finishes the proof by induction.

Taking $j=d$, we obtain from Eq.~\ref{eq:proof_py} and the nature of permutation matrices that
\begin{align}\label{eq:proof-temp}
    \lim_{\tau\to 0^{+}} \widehat{P}_{\sort{}({\hat{y}})}(\tau)_{:k} = \left[ P_{\sort{}({\hat{y}})}\right]_{:k}.
\end{align}
From limit calculus, we know that the limit of finite sums is the sum of the limits and hence, substituting the above result in Eq.~\ref{eq:DCG_soft} we have:
\begin{align}
    \lim_{\tau\to 0^{+}} \widehat{\mathrm{DCG}}(\vy, \hat{\vy}, \tau) = \mathrm{DCG}(\vy, \hat\pi).
\end{align}
Substituting the above in Eq.~\ref{eq:neuralrank} and Eq.~\ref{eq:neuralrank_loss} proves the proposition.

\end{proof}

Note that the assumption of independent draws is needed to ensure that the elements of $\hat{\vy}$ are distinct almost surely.

\begin{table}
\centering
\caption{Shared parameter values for benchmark (Section~\ref{sec:benchmark}) and ablation (Section~\ref{sec:ablation}) experiments.}
\label{tab:expParams}
\begin{tabular}{|l|l|l|} 
\hline
Parameter                & Benchmark & Ablation                   \\ 
\hline
Hidden layer sizes       & 1024,512,256 & 256,256,128,128,64,64  \\
Hidden layer activations & ReLu & ReLu                   \\
Batch normalization     & Yes & No \\
Dropout rate            & 0.3 & 0               \\
Batch size               & 16 & 16                     \\
Learning rate            & 1.00E-03 & 1.00E-05               \\
Optimizer                & Adam & Adam                   \\
Iterations           & 100,000 steps & 100 epochs                    \\
Training list size $L_{train}$       & 200 & 100 when fixed                    \\
Testing list size $L_{test}$        & 200 & 100 when fixed                    \\
Temperature $\tau$ (PiRank \& NeuralSort) & 1000 & 5 when fixed  \\
Straight-through estimation (PiRank \& NeuralSort) & Yes & Yes  \\
NDCG cutoff $k$ (PiRank \& LambdaRank)     & 10 & 10 \\
Depth $d$ (PiRank)     & 1 & 1 when fixed \\
\hline
\end{tabular}
\end{table}

\begin{table}[t]
\centering
\scriptsize
\caption{Training list size effectiveness on ranking metrics}
\label{tab:listOPA}
\begin{tabular}{|r|rrrr|} 
\hline
\multicolumn{1}{|l|}{OPA}  & \multicolumn{1}{l}{$L_{train}$} & \multicolumn{1}{l}{} & \multicolumn{1}{l}{} & \multicolumn{1}{l|}{}  \\ 
\hline
\multicolumn{1}{|l|}{$L_{test}$} & 10                        & 20                   & 40                   & 100                    \\ 
\hline
10                          & 0.5830                    & 0.5947               & 0.5939               & \textbf{0.5949}       \\
20                          & 0.5852                    & 0.5949               & \textbf{0.5961}      & 0.5926                \\
40                          & 0.5816                    & 0.5935               & \textbf{0.5942}      & 0.5915                \\
100                         & 0.5755                    & 0.5859               & \textbf{0.5867}      & 0.5844                \\
\hline
\end{tabular}

\label{tab:listMRR}
\begin{tabular}{|r|rrrr|} 
\hline
\multicolumn{1}{|l|}{MRR}  & \multicolumn{1}{l}{$L_{train}$} & \multicolumn{1}{l}{} & \multicolumn{1}{l}{} & \multicolumn{1}{l|}{}  \\ 
\hline
\multicolumn{1}{|l|}{$L_{test}$} & 10                        & 20                   & 40                   & 100                    \\ 
\hline
10                          & 0.6691                    & 0.6830               & 0.6912               & \textbf{0.6949}       \\
20                          & 0.6835                    & 0.7048               & 0.7087               & \textbf{0.7172}       \\
40                          & 0.6732                    & 0.7042               & 0.7230               & \textbf{0.7350}       \\
100                         & 0.6628                    & 0.6985               & 0.7301               & \textbf{0.7548}       \\
\hline
\end{tabular}

\label{tab:listARP}
\begin{tabular}{|r|rrrr|} 
\hline
\multicolumn{1}{|l|}{ARP}  & \multicolumn{1}{l}{$L_{train}$} & \multicolumn{1}{l}{} & \multicolumn{1}{l}{} & \multicolumn{1}{l|}{}  \\ 
\hline
\multicolumn{1}{|l|}{$L_{test}$} & 10                        & 20                   & 40                   & 100                    \\ 
\hline
10                          & 5.0164           & 4.9584               & 4.9662               & \textbf{4.9428}                \\
20                          & 9.4277           & 9.3431               & \textbf{9.3334}      & 9.3401                         \\
40                          & 18.3042          & 18.0688              & \textbf{18.0493}     & 18.0617                        \\
100                         & 42.9107          & 42.4183              & \textbf{42.3972}     & 42.4091                        \\
\hline
\end{tabular}

\label{tab:listNDCG1}
\begin{tabular}{|r|rrrr|} 
\hline
\multicolumn{1}{|l|}{NDCG@1}  & \multicolumn{1}{l}{$L_{train}$} & \multicolumn{1}{l}{} & \multicolumn{1}{l}{} & \multicolumn{1}{l|}{}  \\ 
\hline
\multicolumn{1}{|l|}{$L_{test}$} & 10                        & 20                   & 40                   & 100                    \\ 
\hline
10                          & 0.3850                    & 0.4127               & 0.4140               & \textbf{0.4261}       \\
20                          & 0.3320                    & 0.3521               & 0.3670               & \textbf{0.3860}       \\
40                          & 0.2829                    & 0.3054               & 0.3403               & \textbf{0.3683}       \\
100                         & 0.2569                    & 0.2665               & 0.3401               & \textbf{0.3713}       \\
\hline
\end{tabular}

\label{tab:listNDCG3}
\begin{tabular}{|r|rrrr|} 
\hline
\multicolumn{1}{|l|}{NDCG@3}  & \multicolumn{1}{l}{$L_{train}$} & \multicolumn{1}{l}{} & \multicolumn{1}{l}{} & \multicolumn{1}{l|}{}  \\ 
\hline
\multicolumn{1}{|l|}{$L_{test}$} & 10                        & 20                   & 40                   & 100                    \\ 
\hline
10                          & 0.4610                    & 0.4793               & 0.4826               & \textbf{0.4878}       \\
20                          & 0.3757                    & 0.3885               & 0.4017               & \textbf{0.4092}       \\
40                          & 0.3188                    & 0.3373               & 0.3572               & \textbf{0.3731}       \\
100                         & 0.2780                    & 0.2963               & 0.3349               & \textbf{0.3579}       \\
\hline
\end{tabular}

\label{tab:listNDCG5}
\begin{tabular}{|r|rrrr|} 
\hline
\multicolumn{1}{|l|}{NDCG@5}  & \multicolumn{1}{l}{$L_{train}$} & \multicolumn{1}{l}{} & \multicolumn{1}{l}{} & \multicolumn{1}{l|}{}  \\ 
\hline
\multicolumn{1}{|l|}{$L_{test}$} & 10                        & 20                   & 40                   & 100                    \\ 
\hline
10                          & 0.5358                    & 0.5498               & 0.5531               & \textbf{0.5570}       \\
20                          & 0.4181                    & 0.4271               & 0.4388               & \textbf{0.4441}       \\
40                          & 0.3447                    & 0.3607               & 0.3780               & \textbf{0.3896}       \\
100                         & 0.2971                    & 0.3158               & 0.3461               & \textbf{0.3635}       \\
\hline
\end{tabular}

\label{tab:listNDCG10}
\begin{tabular}{|r|rrrr|} 
\hline
\multicolumn{1}{|l|}{NDCG@10}  & \multicolumn{1}{l}{$L_{train}$} & \multicolumn{1}{l}{} & \multicolumn{1}{l}{} & \multicolumn{1}{l|}{}  \\ 
\hline
\multicolumn{1}{|l|}{$L_{test}$} & 10                        & 20                   & 40                   & 100                    \\ 
\hline
10                          & 0.6994                    & 0.7100               & 0.7115               & \textbf{0.7141}       \\
20                          & 0.5090                    & 0.5165               & 0.5257               & \textbf{0.5305}       \\
40                          & 0.3989                    & 0.4106               & 0.4243               & \textbf{0.4337}       \\
100                         & 0.3330                    & 0.3485               & 0.3720               & \textbf{0.3878}      \\
\hline
\end{tabular}
\end{table}

\section{Experimental Details}
\label{sec:exp_details}

\label{sec:datasets}

\paragraph{Datasets.} We test \name{} on MSLR-WEB30K\footnote{\label{note:mslr}https://www.microsoft.com/en-us/research/project/mslr/} and the Yahoo! LTR dataset C14\footnote{https://webscope.sandbox.yahoo.com}.
MSLR-WEB30K contains 30,000 queries from Bing with feature vectors of length 136, while Yahoo! C14 dataset comprises 36,000 queries, 883,000 items and feature vectors of length 700.
In both datasets, the number of items per query can exceed 100, or even 1,000 in the case of MSLR-WEB30K.
Both datasets have relevance scores on a 5-point scale of 0 to 4, with 0 denoting complete irrelevance and 4 denoting perfect relevance.
Note that when using binary classification-based metrics such as mean-reciprocal rank, ordinal relevance score from 1 to 4 are mapped to ones. 
MSLR-WEB30K is provided in folds of training / validation / test sets rotating on 5 subsets of data, and we choose to use Fold1 for our experiments.
For Yahoo! C14, we use ``Set 1" which is the larger of the two provided sets.
For both datasets, we use the standard train/validation/test splits.
We use the validation split for both early stopping and hyperparameter selection for all approaches.


\paragraph{TFR Implementation.} We provide a TensorFlow Ranking implementation of the \name{} NDCG Loss as well as the original NeuralSort Permutation Loss which can be plugged in directly into TensorFlow Ranking.\ifanonymous\else\footnote{https://github.com/ermongroup/pirank}\fi


\paragraph{Straight-through Estimation.} The \name{} surrogate learning objective can be optimized via two gradient-based techniques in practice.
The default mode of learning is to use the relaxed objective during both forward pass for evaluating the loss and for computing gradients via backpropogation.
Alternatively, we can perform \textit{straight-through estimation}~\cite{bengio2013estimating}, where we use the hard version for evaluating the loss forward, but use the relaxed objective in Eq.~\ref{eq:neuralrank} for gradient evaluation.
We observe improvements from the latter option and use it throughout.
The hard version can be obtained via exact sorting of the predicted scores.
In the context of a unimodal relaxation (Sec~\ref{sec:algo}), a hard version can also be obtained via a row-wise arg max operation of the relaxed permutation matrix, which recovers an actual permutation matrix usable in the downstream objective.


\paragraph{Architecture and Parameters.} Experiment parameters that are shared across losses, such as the scoring neural network architecture, batch size, training and test list sizes, are provided in Table~\ref{tab:expParams}, along with loss-specific parameters if they differ from the default TensorFlow Ranking setting.


\paragraph{Experimental Workflow.} We rely on TensorFlow Ranking for most of our work outside the NeuralSort and \name{} loss implementations, which takes care of query grouping, document list tensor construction, baseline implementation and metric computation among others.

\paragraph{Computing infrastructure.} The experiments were run on a server with 4 8-core Intel Xeon E5-2620v4 CPUs, 128 GB of RAM and 4 NVIDIA Telsa K80 GPUs.


\paragraph{Libraries and Software.} This work extensively relied on GNU Parallel \cite{Tange2011a} and the Sacred library \footnote{https://github.com/IDSIA/sacred} for experiments.

\paragraph{Licenses.} 
TensorFlow Ranking is licensed under the Apache License 2.0 \footnote{https://github.com/tensorflow/ranking/blob/master/LICENSE}. GNU Parallel is licensed under the GNU General Public License \footnote{https://www.gnu.org/licenses/gpl-3.0.html}. Sacred is licensed under the MIT License \footnote{https://github.com/IDSIA/sacred/blob/master/LICENSE.txt}. The dataset MSLR-WEB30K is licensed under the Microsoft Research License Agreement (MSR-LA). The license files for the dataset Yahoo! C14 are provided in the datasets at download time from their homepages, and included in the supplemental material.
Our released \name{} code is licensed under the MIT license. 

\section{Ablation Experiments}
\label{sec:size_exps}

We provide all results for the temperature experiments described in Section~\ref{sec:ablation}, in Figures~\ref{fig:tempndcg1}, \ref{fig:tempndcg3}, \ref{fig:tempndcg5}, \ref{fig:tempndcg15}, \ref{fig:temptrainloss}, \ref{fig:temperature}. The smoothing parameter used in all plots is 0.9, the number of data points is 100 epochs for all figures except the training loss (1,000 iterations).
We provide additional plots with a basic exponentially decreasing annealing schedule and a very high temperature value of 1e12 to show limits of the relaxation on Figure~\ref{fig:newtemps}.
Full results for the ablation experiments on the training list size described in Section~\ref{sec:ablation} are provided in Table~\ref{tab:listMRR}.

\section{Synthetic LTR Data}
\label{sec:synthetic}

To the best of our knowledge, there is no public LTR dataset with very large numbers of documents per query ($L > 1000$). We thus propose the following synthetic dataset for testing and development at scale (see Section~\ref{sec:depthablation}):

For each query $q_i$, $i \in \{1, \cdots, n\}$,
\begin{enumerate}
    \item Generate $L$ documents $\{\mathbf{x}_{i,j}\}_{j=1}^L$ where $\mathbf{x}_{i,j}$ is a vector of $m_d$ $\phi$-distributed document features.
    \item Randomly pick a vector $c_i$ of $m_q \leq m_d$ column indices from $\{1, \cdots, m_d\}$ without replacement.
    \item Generate $\psi$-distributed query features $\{\gamma_i\}_{k=1}^{m_q}$.
    \item Compute labels capped between $\ell$ and $h$ s.t. \\ $y_{i,j} = max(\ell, min(h, \sum_{k=1}^{m_q}\gamma_k x_{i,j,c_k} ))  $.
    \item Concatenate the query features $\{\gamma_i\}_{k=1}^{m_q}$ to each $\mathbf{x}_{i,j}$.
\end{enumerate}

This process allows us to generate datasets of arbitrarily large size, where we control $L$, $n$, $m$, $c$ and the distributions $\phi$ and $\psi$.
The process is easy to reuse, and made available in our TFR codebase.


\begin{figure}[h]
    \centering
    \includegraphics[width=0.65\textwidth]{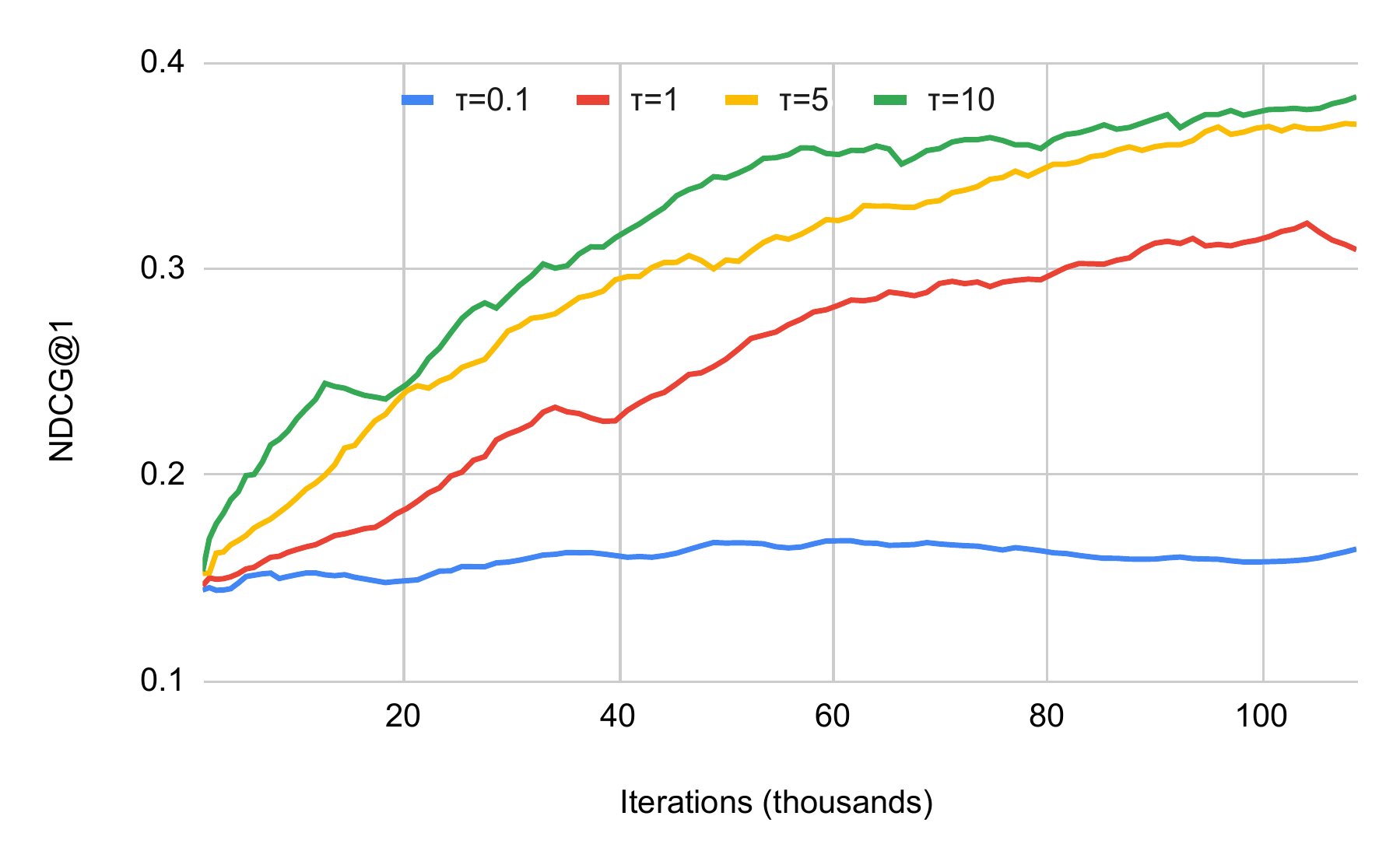}
    \caption{Validation NDCG@1 during \name{} training parametrized by temperature $\tau$}
    \label{fig:tempndcg1}
\end{figure}

\begin{figure}[h]
    \centering
    \includegraphics[width=0.65\textwidth]{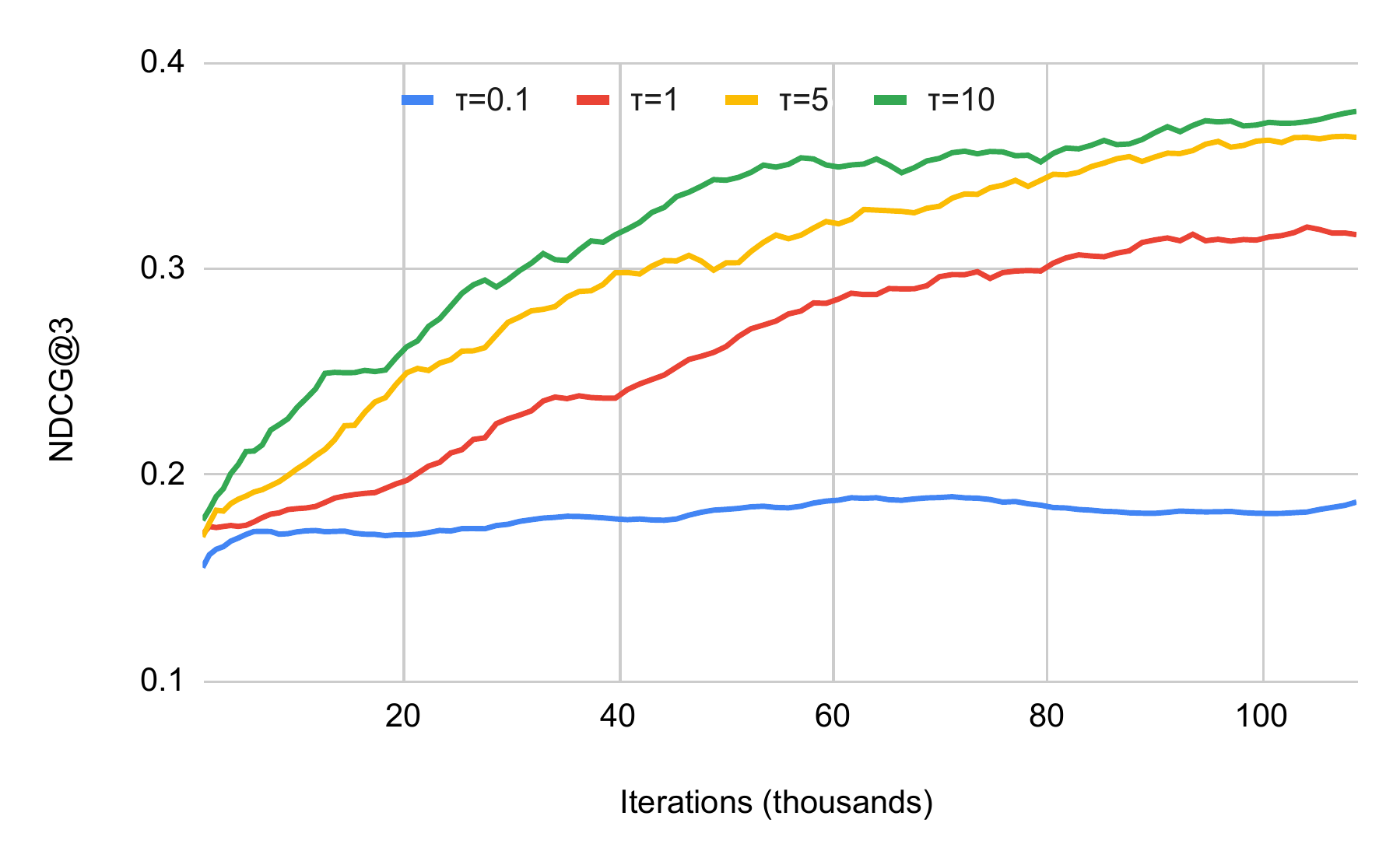}
    \caption{Validation NDCG@3 during \name{} training parametrized by temperature $\tau$}
    \label{fig:tempndcg3}
\end{figure}

\begin{figure}[h]
    \centering
    \includegraphics[width=0.65\textwidth]{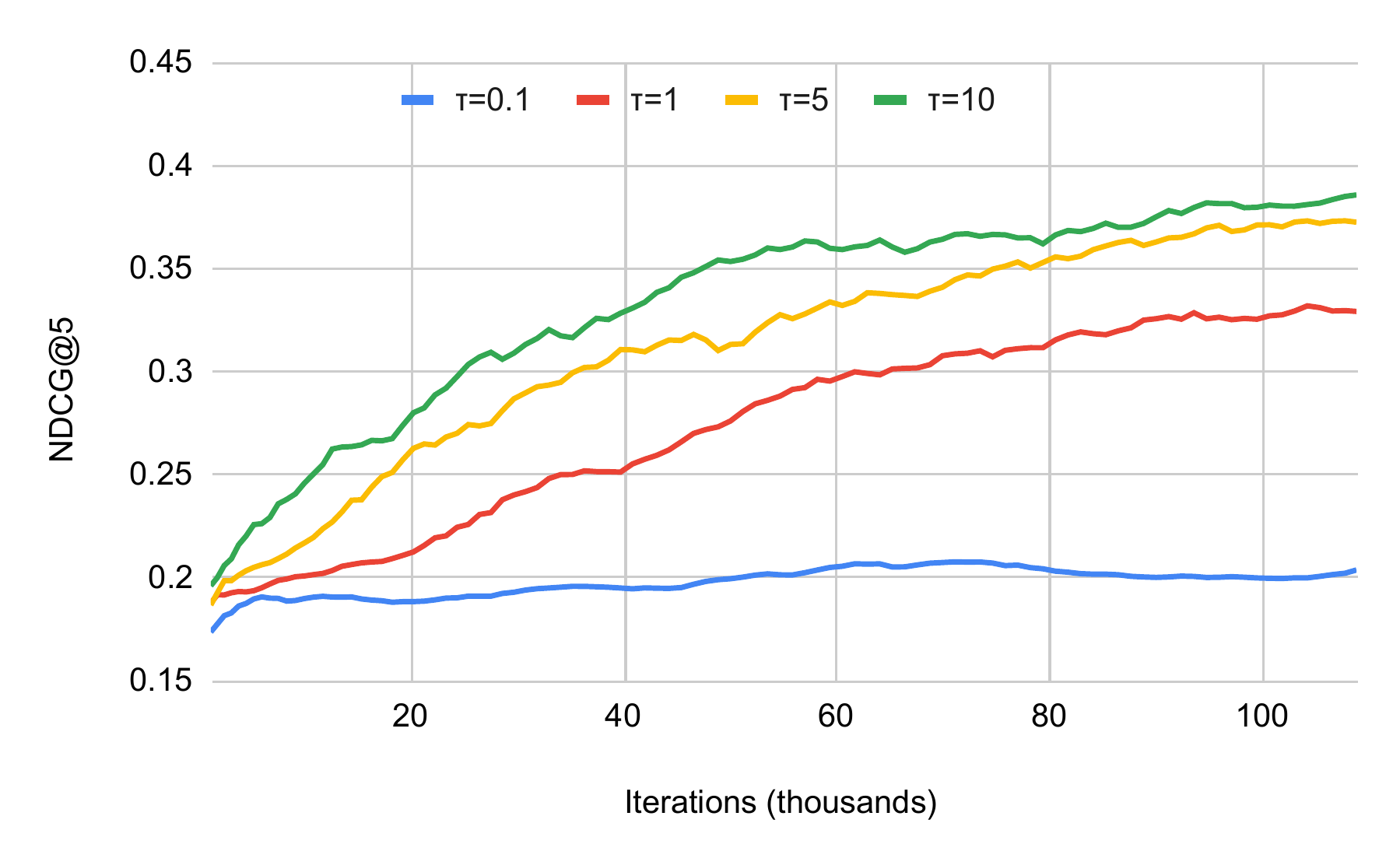}
    \caption{Validation NDCG@5 during \name{} training parametrized by temperature $\tau$}
    \label{fig:tempndcg5}
\end{figure}

\begin{figure}[h]
    \centering
    \includegraphics[width=0.65\textwidth]{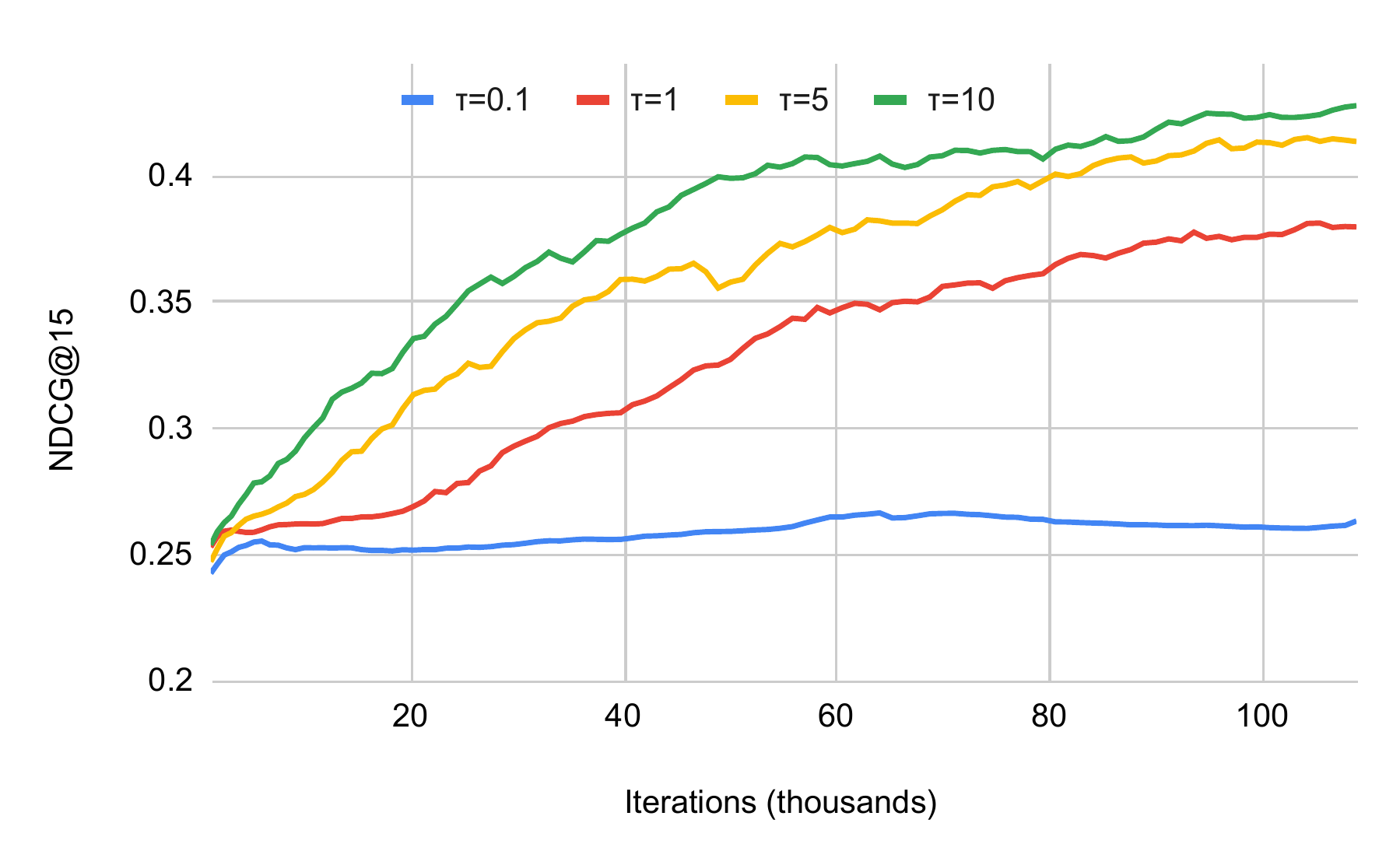}
    \caption{Validation NDCG@15 during \name{} training parametrized by temperature $\tau$}
    \label{fig:tempndcg15}
\end{figure}

\begin{figure}[h]
    \centering
    \includegraphics[width=0.65\textwidth]{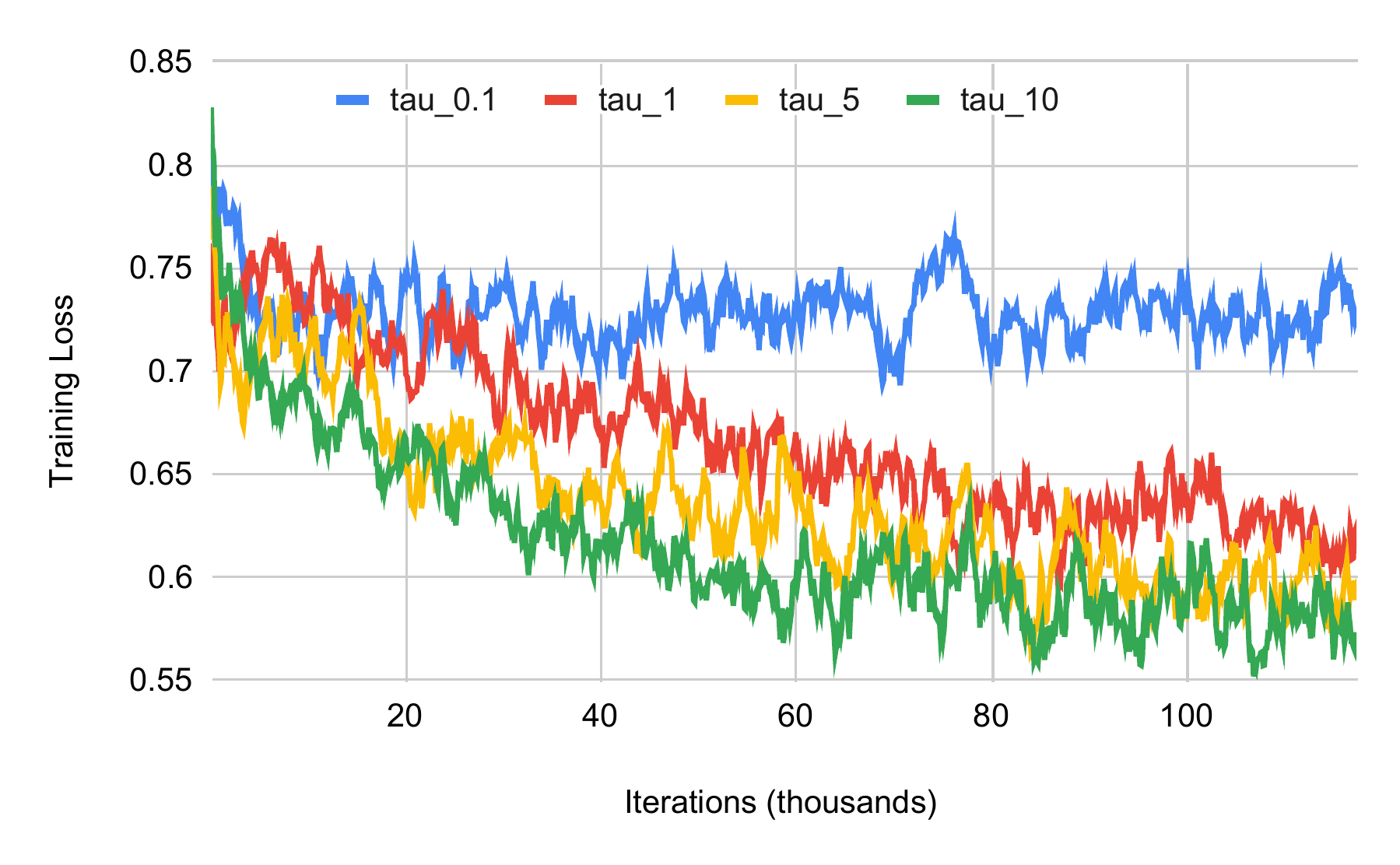}
    \caption{Training loss during \name{} training parametrized by temperature $\tau$}
    \label{fig:temptrainloss}
    \vspace{4in}
\end{figure}

\begin{figure}[h]
    \centering
    \includegraphics[width=0.9\textwidth]{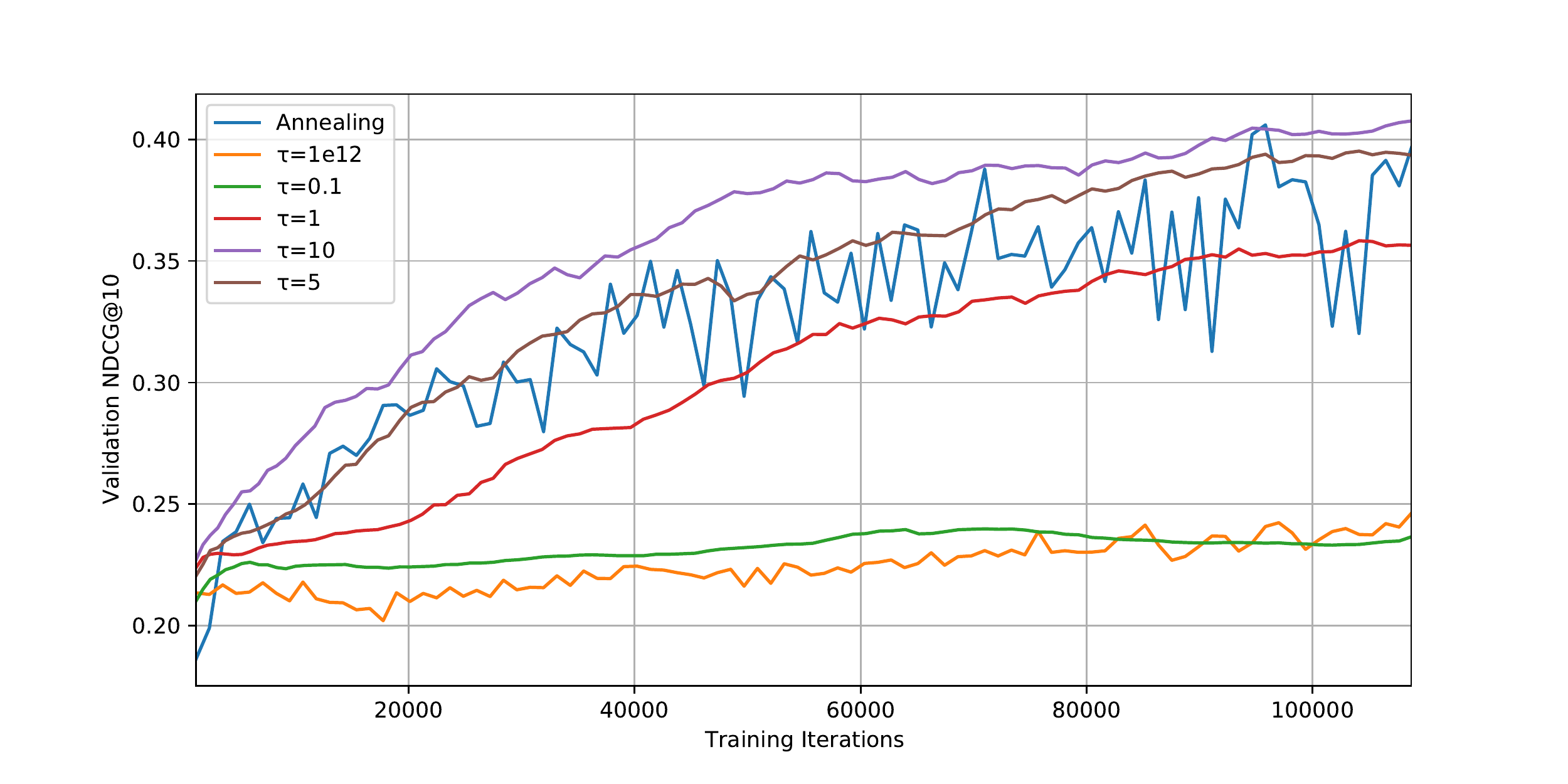}
    \caption{Validation NDCG@10 for PiRank-NDCG@10 using the experimental settings of Section~\ref{sec:ablation}. This figure shows the validation NDCG@10 from Figure~\ref{fig:temperature} superimposed with an annealing schedule temperature (blue) and a very high temperature of 1e12 (orange).}
    \label{fig:newtemps}
    \vspace{4in}
\end{figure}


\end{document}